\documentclass[runningheads]{llncs}
\usepackage[T1]{fontenc}
\usepackage[subtle]{savetrees}
\usepackage{booktabs}

\newcommand{\whittle}{m}
\newcommand{\parameter}{P}

\newcommand{\tmpstate}{u}
\newcommand{\s}{s}
\newcommand{\nstates}{M}
\newcommand{\St}{\mathcal{S}}
\newcommand{\action}{a}
\newcommand{\Action}{\mathcal{A}}
\newcommand{\reward}{r}

\newcommand{\policy}{\pi}
\newcommand{\trajectories}{\mathcal{T}}
\newcommand{\trajectory}{\tau}

\usepackage{algorithm}
\usepackage{amsmath}
\usepackage{amssymb}
\usepackage[noend]{algcompatible}

\usepackage{graphicx}
\begin{document}
\title{IRL for Restless
Multi-Armed Bandits with Applications in Maternal and Child Health}
%
%
\author{Gauri Jain \inst{1} \and
Pradeep Varakantham \inst{2} \and
Haifeng Xu \inst{3} \and Aparna Taneja \inst{4} \and Prashant Doshi \inst{5} \and Milind Tambe \inst{1,4}}
\authorrunning{G. Jain et al.}
%
\institute{Harvard University \and
Singapore Management University \and University of Chicago \and Google Deepmind \and University of Georgia
}

\maketitle              
\begin{abstract}
Public health practitioners often have the goal of monitoring patients and maximizing patients' time spent in ``favorable" or healthy states while being constrained to using limited resources. Restless multi-armed bandits (RMAB) are an effective model to solve this problem as they are helpful to allocate limited resources among many agents under resource constraints, where patients behave differently depending on whether they are intervened on or not. However, RMABs assume the reward function is known. This is unrealistic in many public health settings because patients face unique challenges and it is impossible for a human to know who is most deserving of any intervention at such a large scale. To address this shortcoming, this paper is the first to present the use of inverse reinforcement learning (IRL) to learn desired rewards for RMABs, and we demonstrate improved outcomes in a maternal and child health telehealth program. First we allow public health experts to specify their goals at an aggregate or population level and propose an algorithm to design expert trajectories at scale based on those goals. Second, our algorithm WHIRL uses gradient updates to optimize the objective, allowing for efficient and accurate learning of RMAB rewards. Third, we compare with existing baselines and outperform those in terms of run-time and accuracy. Finally, we evaluate and show the usefulness of WHIRL on thousands on beneficiaries from a real-world maternal and child health setting in India. We publicly release our code here: https://github.com/Gjain234/WHIRL.
\end{abstract}
\begin{figure}
    \centering
    \includegraphics[width=0.63\linewidth]{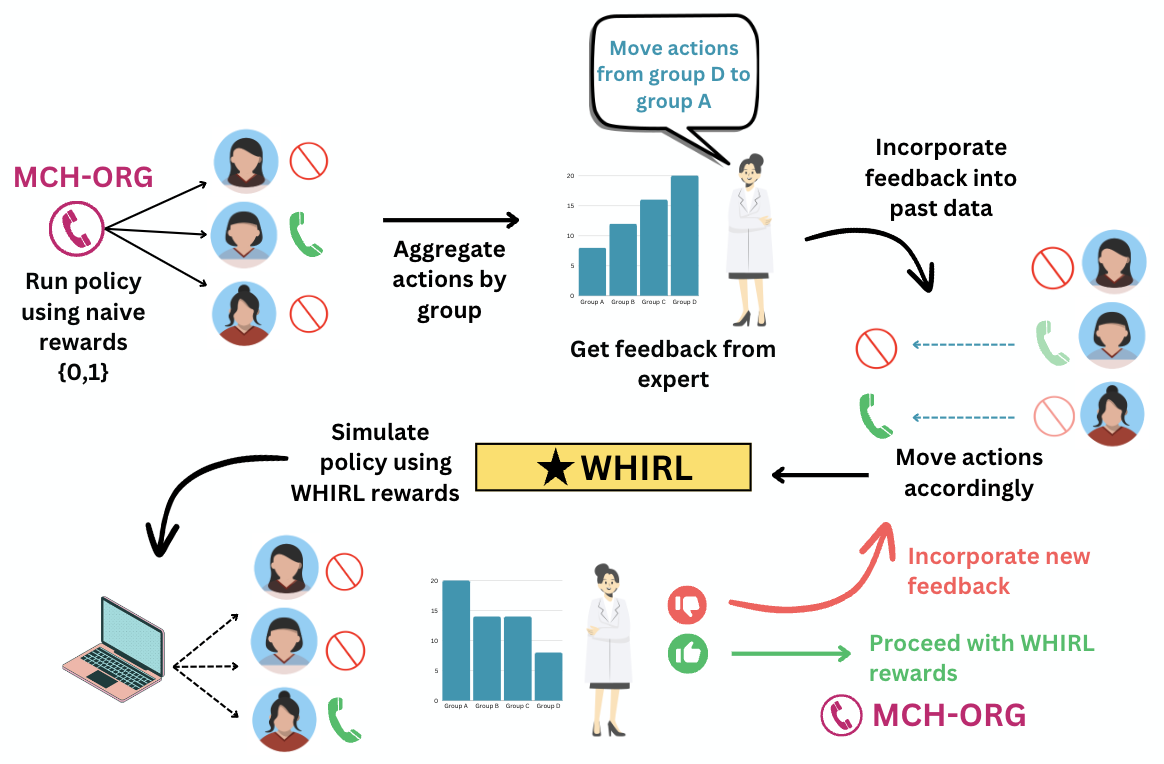}
    \caption{The full system including stakeholders. More detail in Section \ref{applied}}
    \label{fig:system}
\end{figure}
\section{Introduction}
Our work is motivated by challenges faced by Armman, a maternal and child health (MCH) nonprofit in India. Armman delivers telehealth care to beneficiaries during and after their pregnancy. The telehealth program delivers weekly automated voice messages (AVMs) with health information. Since some beneficiaries tend to stop listening to the AVMs, Armman also provides human service calls (HSCs) to encourage beneficiaries to keep listening to the AVMs. Unfortunately, the number of human callers (few tens) is much lower than the number of beneficiaries (few thousands), requiring optimal allocation of the limited HSCs to beneficiaries. Past work models this as an RMAB~\cite{weber1990index} problem where each arm is a beneficiary \cite{mate2022field} and has two states: whether they have \textit{listened} or \textit{not listened} to the AVM last week, resulting in rewards 1 or 0 respectively. The actions at each time step (each week) are a subset of the beneficiaries chosen to receive HSCs. 

A drawback of this homogeneous binary reward design is that it leads to policies that maximize the number of beneficiaries in the \textit{listened} state but ignores the specific circumstances of each beneficiary. After close interviews with members of Armman, we learned that the algorithm was giving similar priority to mothers at very different risk levels for pregnancy complication (Figure \ref{fig:before_after_actions_risk}a), which is counterintuitive because lower risk mothers have many resources and social networks they can rely on, so achieving high listening from them is significantly less important \cite{socialinequality}. This motivated us to change the current $\{0,1\}$ rewards. But risk is just one metric public health experts use to make decisions for their beneficiaries, and we quickly found that there was no clear way to manually design rewards for all of their goals. Additionally, we found that manual reward design methods can end up being even more unfair (Section \ref{applied}). Furthermore, public health experts may want to alter interventions depending on factors that change over time (e.g., education level, health risk). Therefore, we seek to learn rewards for an RMAB that can capture goals of public health experts at any given time. 

The above problem naturally fits IRL \cite{Russell98:Learning}, which learns rewards from expert demonstrations, thereby aligning the planner's sequential actions with desired outcomes. However, our problem is different from classic IRL in two crucial aspects. First, classic IRL predominantly focuses on single-agent MDPs whereas our setup has {\em thousands} of agents (beneficiaries) active simultaneously. Second, typical IRL algorithms are designed to learn from full expert trajectories, but our application's large scale makes it impossible for a human expert to completely specify the entire policy trajectories.

Our work (Figure~\ref{fig:system}) addresses these shortcomings with the following contributions: (i) we design Whittle-IRL \textbf{(WHIRL)}, the first IRL algorithm for learning RMAB rewards; (ii) we propose a novel approach for public health experts to specify aggregate behavior goals and design an algorithm to generate trajectories at a massive scale based on those goals; (iii) we show that WHIRL outperforms current IRL baselines; (iv) we apply WHIRL to a real-world MCH setting with thousands of beneficiaries and show that its use results in outcomes significantly more aligned with health expert goals.

The significance of our work extends beyond this setting since RMABs are also used to model Hepatitis treatment \cite{ayer2019prioritizing}, liver cancer screening \cite{10.1287/msom.2017.0697}, Asthma care \cite{deo_improving_2013}, and more \cite{villar_multi-armed_2015,mate2020collapsing}. Even beyond the health setting, any RMAB optimization problem can benefit from our method. We hope our new method can be helpful to the broader community of researchers using RMABs to make decisions more aligned with human experts.

\section{Related Work}
\label{sec:related}
Some of the best known early approaches for IRL include maximum entropy IRL~\cite{ziebart_maximum_nodate,arora_survey_2021}, which infers the underlying reward function that makes the least commitment to any particular trajectory, and Bayesian IRL techniques such as the fast MAP-BIRL \cite{Choi2011MAPIF}, which uses gradient updates to maximize the {\em a posteriori} likelihood witH trajectories serving as observations. These approaches however grow exponentially as the number of agents increase, so they would not work in our setting. There are also recent deep learning algorithms such as AIRL~\cite{airl} that employ an adversarial approach and stochastic forward RL techniques such as PPO, but they are known to not scale beyond tens of agents~\cite{yu2022surprising}. Previous work in multi-agent IRL focuses on problems with few agents \cite{5708862,bogert2014multi}, assumes agents are homogeneous reducing the problem to a single agent \cite{sosic2017inverse}, or learns from a human-robot system ~\cite{yu2019multi,Suresh23:DecAIRL}. The RMAB's budget constraint induces a challenging multiagent interaction that is not easily modeled by existing approaches. 

IRL has also been proposed for health settings, e.g.,
to make decisions on ventilator units and sedatives in ICUs \cite{yu_inverse_2019} or model diabetes treatments \cite{bml21}, but these have not addressed the RMAB scalability challenges. \cite{wang_scalable_2023} provides a decision-focused learning approach establishing the Whittle policy's differentiability to optimize the RMAB solution quality directly. However, we look to apply gradient-based updates to rewards instead of transition probabilities and therefore must also define a new evaluation function to directly optimize. Furthermore, WHIRL provides a framework for enabling health experts to provide feedback on a large scale to be converted to expert trajectories; such interaction with human experts has been absent in previous RMAB work. 

\section{Methodology}
\subsubsection{Preliminaries}
An RMAB is composed of $N$ arms where each follows an independent MDP. Our state space is discrete with $|\St| = \nstates$ states, and actions $\Action = \{ 0,1\}$ corresponding to not pulling or pulling the arm. $P_i(s,a,s'): \St \times \Action \times \St \rightarrow P$ defines the probability distribution of arm~$i$ in state $s$ transitioning to next state $\s' \in \St$. $R_i(s)$ is the reward function for arm~$i$ in state $s$. At each time step $h \in [H]$, the planner observes $\boldsymbol\s_h = [ \s_{h,i} ]_{i \in [N]} \in \St^N$, the states of all arms, and then chooses action ${\boldsymbol \action}_h = [\action_{h,i}]_{i \in [N]} \in \Action^N$ denoting the actions on each arms, which has to satisfy a budget constraint $\sum\nolimits_{i \in [N]} \action_{h,i} \leq K$.
Once the action is chosen, arms receive actions $\boldsymbol\action_{t}$ and transitions under $P$. The total reward is defined as $\sum\nolimits_{h = 1}^H \gamma^{h-1} \sum\nolimits_{i \in [N]} \reward_{h,i}$, where $0 < \gamma \leq 1$ is the discount factor. A policy is denoted by $\policy$ where $\policy(\boldsymbol\action \mid \boldsymbol \s)$ determines actions taken at joint state of arms. $\policy^\text{learner}$ denotes the policy being learned in WHIRL.

\paragraph{The Whittle Index}
The dominant approach to solve the RMAB problem is the Whittle index policy~\cite{whittle1988restless}. 
The Whittle Index evaluates the value of pulling an arm $i$ via a `passive subsidy', i.e., a hypothetical compensation $\whittle$ rewarded for not pulling the arm (i.e., for action $a=0$). Given state $\tmpstate \in \St$, we define the Whittle index associated to state $\tmpstate$ by:
\begin{equation}
    W_i(\tmpstate) = \inf\nolimits_{\whittle}
    \{ Q_{i}^\whittle(\tmpstate; a=0) = Q_{i}^\whittle(\tmpstate; a=1)\} \label{eqn:action-indifference}
\end{equation}
\begin{equation}
     V^{\whittle}_i(s) = \textstyle \max\nolimits_{\action} Q_{i}^\whittle(\s; \action)  \label{eqn:bellman-equation}
\end{equation}
\begin{equation}
    \textstyle Q_{i}^\whittle(\s; \action) = \textstyle \whittle \boldsymbol 1_{a = 0} + R(\s) + \gamma \sum\nolimits_{\s'} \parameter_{i}(\s,a,\s') V^{\whittle}_{i}(s') \label{eqn:bellman-equation2}
\end{equation}
The subsidy $\whittle$ is added for action $a=0$. The Whittle policy then selects arms to pull: $\policy_W(\boldsymbol\s) = \boldsymbol1_{\text{top-k}([W_i(\s_i)]_{i \in [N]})} \in \{0,1\}^N$.

\paragraph{Soft Top K}
 For our setting, we use a soft-top-k policy selection, which gives the probability of pulling each arm. In addition to its differentiability property, soft-top-k is a better policy for our setting since it captures the bounded rationality of human planners.
\begin{align}\label{eqn:differentiable-whittle-policy}
    \textstyle \policy^{\text{soft}}_W(\boldsymbol\s) = \text{soft-top-k}([W_j(\s_i)]_{i \in [N]}) \in [0,1]^N
\end{align}

\paragraph{IRL}
In IRL, the learner does not know the reward function $\mathbf{R}(\s)$ in advance. Each arm $i$ can have a different reward in each state. The goal is to learn rewards that best fit a set of $J$ expert trajectories $\trajectories = \{\trajectory^{(j)}\}_{j \in J}$, i.e., maximize $P(\trajectories|\policy^\text{learner},\Tilde{\mathbf{R}})$ where $\Tilde{\mathbf{R}}$ is the reward being learned for all beneficiaries. We denote a full trajectory over $H$ timesteps by $\trajectory = (\boldsymbol\s_1, \boldsymbol\action_1,  \cdots, \boldsymbol\s_H, \boldsymbol\action_H)$, where $\boldsymbol\s, \boldsymbol\action$ are the joint state and action of all $N$ arms.

\subsubsection{Problem Statement}
\label{sec:application_public_health}
Formally, we consider an RMAB setup where at timestep $H$ for trajectory $\trajectory$, a human expert gives an aggregate level directive: ``Move interventions from category $\mathbf{{C}}^{\tau}_{h,A} = \{ \text{arm} | \text{arm} \in {f}_\text{source}(h) \}$ to $\mathbf{{C}}^{\tau}_{h,B} = \{ \text{arm} | \text{arm} \in {f_\text{target}(h)} \}$.'' These categories represent combinations of static features about beneficiaries like their income, education, language spoken, etc., \textit{and/or} dynamic features like the current state they are in or last time they were acted on. According to the directive, we look at the past trajectory data and at each previous timestep $h=1, \cdots, H$, we allocate service calls from a subset of beneficiaries $B_{A} \subset \mathbf{{C}}^{\tau}_{h,A}$  to a subset $B_{B} \subset \mathbf{{C}}^{\tau}_{h,B}$ to generate expert trajectories $\trajectories^\text{expert}$. We then learn the underlying RMAB reward estimation $\Tilde{\mathbf{R}}$ by maximizing $P(\trajectories^\text{expert}|\policy^\text{learner},\Tilde{\mathbf{R}})$. From timestep $H$ onward, we allocate service calls using an RMAB policy generated by the new $\Tilde{\mathbf{R}}$.

\subsubsection{IRL Approach for RMAB: WHIRL}
\begin{figure*}
    \centering
    \includegraphics[width=0.9\linewidth]{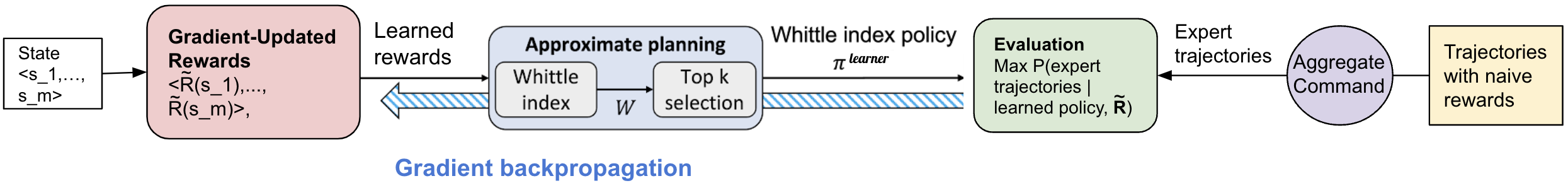}
    \caption{WHIRL iterates through the Whittle policy (blue) to estimate the final evaluation (green) and run gradient ascent to update rewards (red) using Eqn \ref{eqn:differentiable-whittle-policy-derivative}. In yellow, $\trajectories^\text{expert}$ is generated from an aggregate directive (purple) and used in  evaluation.}
    \label{fig:reward_dfl}
\end{figure*}
We propose WHIRL (Algorithm \ref{alg:approximate-decision-focused-learning}), which first generates trajectories, $\trajectories^\text{expert}$ from aggregate desired behavior by calling Algorithm \ref{alg:population-edits} (line 2). It then calculates Whittle indices using the known transition probabilities and estimated rewards $\Tilde{\mathbf{R}}$ and generates a soft top k policy on the indices (line 5), which is referred to as $\pi^\text{learner}$. We then pass that policy to our evaluation function that evaluates the likelihood of seeing $\trajectories^\text{expert}$ using $\Tilde{\mathbf{R}}$. WHIRL keeps track of the gradient with respect to reward throughout so that we can 
 perform updates on $\Tilde{\mathbf{R}}$. Figure \ref{fig:reward_dfl} visualizes this sequence. We now go in depth into each part of the algorithm.

\begin{algorithm}[tb]
   \caption{WHIRL}
   \label{alg:approximate-decision-focused-learning}
\begin{algorithmic}[1]
   \STATE {\bfseries Input:} $\trajectories, \mathbf{P},$  learning rate $r$, aggregate directive. {\bfseries Initialize:} $\tilde{\mathbf{R}} = \mathbf{0}$
   \STATE $\trajectories^\text{expert} = \Big\{\text{Algorithm}~\ref{alg:population-edits}({f_\text{source}}, {f_\text{target}} , \tau$)\Big\}$_{\tau \in \trajectories}$
   \FOR{epoch $= 1,2,\cdots$}
   \STATE Compute Whittle indices $W(\tilde{\mathbf{R}})$. 
   \STATE Let $\policy^\text{learner} = \policy^{\text{soft}}_W$ and compute $\text{Eval}(\policy^\text{learner}, \trajectories^\text{expert})$.
   \STATE Update $\tilde{\mathbf{R}} = \tilde{\mathbf{R}} + \alpha \frac{d \text{Eval}(\policy^\text{learner}, \trajectories^\text{expert})}{d \policy^\text{learner}} \frac{d \policy^\text{learner}}{d W} \frac{d W}{d \tilde{\mathbf{R}}}$.
   \ENDFOR
   \noindent\STATE {\bfseries Return:} reward $\tilde{\mathbf{R}}$
\end{algorithmic}
\end{algorithm}
\paragraph{Convert Aggregate Behavior to Trajectories} \label{sec:aggregate_edit}

Because experts cannot generate per-arm trajectories for thousands of arms, we use desired \textit{aggregate behavior}, which is specified as moving interventions from some source categories of arms to target categories of arms in the observed trajectory $\tau$. We propose Algorithm \ref{alg:population-edits} which modifies $\tau$ to achieve the desired behavior specified by configurable functions $f_\text{source}$ and $f_\text{target}$, where the former defines which eligible beneficiaries to have interventions moved from and the latter defines where to re-allocate those interventions. $f_\text{source}$ and $f_\text{target}$ can handle complex logic related to beneficiaries (see Table \ref{tab:systematic} for examples). We work closely with public health experts when defining these functions. Implicitly, Algorithm \ref{alg:population-edits} also has access to beneficiary static and dynamic features to determine which beneficiaries qualify to be moved at a given timestep. We iterate over all $H$ and compute the set of arms in source category $\mathbf{{C}}^{\tau}_{h,A}$ and target category $\mathbf{{C}}^{\tau}_{h,B}$ at each step. We use the principle of maximum entropy~\cite{jaynes}, which in our case means having an equal probability of allocating any action from $\mathbf{{C}}^{\tau}_{h,A}$ to $\mathbf{{C}}^{\tau}_{h,B}$ (Theorem \ref{maxentropy}). Because of the algorithm's built-in entropy, we can rerun this algorithm multiple times to obtain many different trajectories for training WHIRL. The significance of this algorithm is that it is able to convert desired aggregate behavior into expert trajectories for all arms. This allows the application of IRL to an RMAB instance with thousands of arms and trajectories in a way that was not previously possible. 
\begin{theorem} \label{maxentropy}
Algorithm \ref{alg:population-edits} allocates actions with max entropy.
\end{theorem}
\begin{proof}
Assume Algorithm \ref{alg:population-edits} does not satisfy maximum entropy. This means that there is built-in bias where $\exists $ some $ \trajectory^a$ and $\trajectory^b$ where $P(\trajectory^b|\policy) \neq P(\trajectory^a|\policy)$. If we designate the probability of a $\trajectory^x$ occurring as $\prod_{t=1}^H p_t^x$ where $p_t$ is the probability of the allocation decisions at each time step, this means there must be at least one $t$ where $p^a_t \neq p^b_t$. However, this is impossible because at each $t$, every allocation decision is made with an even probability by line 4 of Algorithm \ref{alg:population-edits}, so our original assumption must be false.
\end{proof}
\begin{algorithm}[tb]
   \caption{Convert Aggregate Behavior to Trajectory}
   \label{alg:population-edits}
\begin{algorithmic}[1]
   \STATE {\bfseries Input:} ${f_\text{source}}, {f_\text{target}}, \trajectory$
   \FOR{h $= 1,2,\cdots H$}
       \STATE$\mathbf{{C}}^{\tau}_{A,h} = \{c_A|c_A \in {f_\text{source}}(\text{h}), a_h^{c_A,\tau} = 1\} $ , $\mathbf{{C}}^{\tau}_{B,h} = \{c_B|c_B \in {f_\text{source}}(\text{h}), a_h^{c_B,\tau} = 0 \}$ 
       \STATE $\text{arm\_selection\_prob} = 
       \frac{|\mathbf{{C}}^{\tau}_{B,h}|}{|\mathbf{{C}}^{\tau}_{A,h}|}$
       \FOR{$c_B \in \mathbf{{C}}^{\tau}_{h,B}$}
            \algorithmicif { $\mathbf{C}^{\trajectory}_{A,h}\neq\varnothing$}
               \STATE $a_h^{c_B,\tau}=1$, $a_h^{c_A,\tau}=0$ for a random $c_A \in \mathbf{{C}}^{\tau}_{A,h}$ 
               \STATE $\mathbf{{C}}^{\tau}_{A,h}\text{.pop}(c_A)$
           \
       \ENDFOR
   \ENDFOR
   \STATE {\bfseries Return:} $\tau$ with new aggregate behavior
\end{algorithmic}
\end{algorithm}

\paragraph{Policy Evaluation Function}\label{sec:policy-evaluation}

We now turn to  \textit{policy evaluation}, the green box in Figure \ref{fig:reward_dfl}. We employ a maximum likelihood estimation (MLE) evaluation approach~\cite{arora_survey_2021}. Given a set of expert trajectories $\trajectories^\text{expert}$ of size $H$, we estimate $P(\trajectories^\text{expert}|\policy^\text{learner},\Tilde{\mathbf{R}})$ and use this evaluation function to update the reward parameter via one step of gradient ascent. 
To derive the MLE evaluation function, we use $s^{i,\tau}_h$ and $a^{i,\tau}_h$ to denote the state and action of arm $i$ on trajectory $\tau$ at timestep $h$. $P(s^{i,\tau}_h, a^{i,\tau}_h, s^{i,\tau}_{h+1})$ is the transition probability $P_i$. $P(a^{i,\tau}_h|\mathbf{s^{\tau}_{h}},\policy^\text{learner})$ is the soft-top-k probability for pulling arm $i$ given $\policy^\text{learner}$ which is generated from the learned rewards in that iteration (Equation \ref{eqn:differentiable-whittle-policy}). We take the log of this value since it preserves ordering but is differentiable. Lastly, since the transition probabilities, we can drop them from the optimization and be left with Equation \ref{eqn:simplified_eval}.

\begin{align}
    &\textstyle \text{Eval}(\policy^\text{learner}, \trajectories^\text{expert}) = P(\trajectories^\text{expert}|\policy^\text{learner},\Tilde{\mathbf{R}}) \propto \log P(\trajectories^\text{expert}|\policy^\text{learner},\Tilde{\mathbf{R}}) \label{log_prop} \\ 
    &\textstyle = \sum^{\tau \in \trajectories^\text{expert}}_{i \in N} \sum_{h=1}^H \Big[\log(P(s^{i,\tau}_h, a^{i,\tau}_h, s^{i,\tau}_{h+1})) + \log(P(a^{i,\tau}_h|\mathbf{s^{\tau}_{h}},\policy^\text{learner}))\Big] \label{eqn:log-eval} \\
    &\textstyle \propto \sum^{\tau \in \trajectories^\text{expert}}_{i \in N} \sum_{h=1}^H \log(P(a^{i,\tau}_h|\mathbf{s^{\tau}_{h}},\policy^\text{learner})) \label{eqn:simplified_eval}
\end{align}

\paragraph{Gradient Updates}
The arm-wise gradient update method for updating the estimated reward is key to WHIRL’s scalability. To apply it, we need to compute $\frac {d \text{Eval}}{d \Tilde{R}}$.
\begin{align}\label{eqn:differentiable-whittle-policy-derivative}
    \textstyle \frac{d \text{Eval}(\policy^\text{learner}, \trajectories^\text{expert})}{d \Tilde{\mathbf{R}}} = \frac{d \text{Eval}(\policy^\text{learner}, \trajectories^\text{expert})}{d \policy^\text{learner}} \frac{d \policy^\text{learner}}{d W} \frac{d W}{d \Tilde{\mathbf{R}}}
\end{align}
where $W$ is the Whittle indices of all states under learned rewards $\Tilde{\mathbf{R}}$.  $\frac{d \text{Eval}(\policy^\text{learner}, \trajectories^\text{expert})}{d \policy^\text{learner}}$ can be directly computed since it is a sum of differentiable log functions (Equation \ref{eqn:simplified_eval}), and $\frac{d \policy^\text{learner}}{d W}$ was shown to be differentiable in \cite{wang_scalable_2023}. To show differentiability through Whittle index computation to derive $\frac{d W}{d \Tilde{\mathbf{R}}}$, we can express whittle indices as a linear function of $\Tilde{\mathbf{R}}$ as done in \cite{wang_scalable_2023} allowing for the computation of $\frac{d W}{d \Tilde{\mathbf{R}}}$ via auto differentiation. 
\label{dwdr} 
\begin{figure*}
    \centering
    \includegraphics[width=1\linewidth]{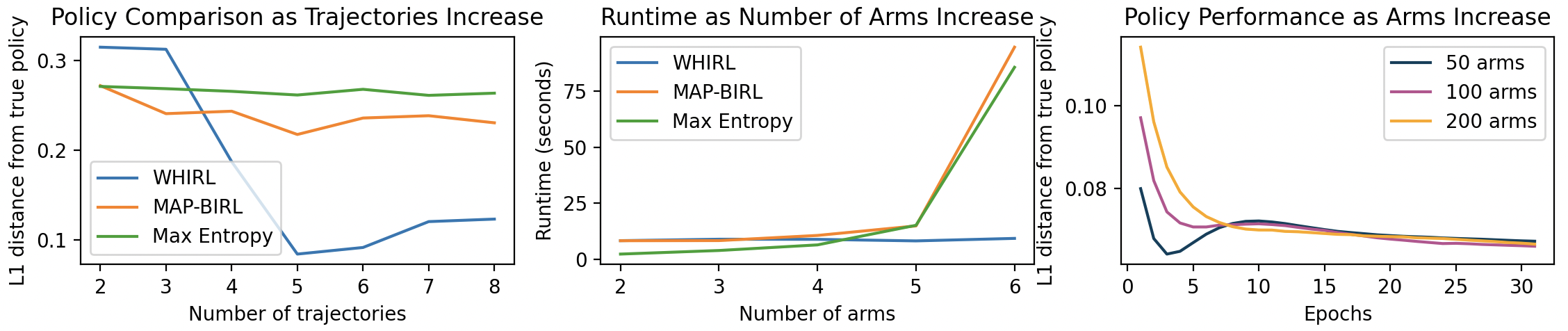}
    \caption{
    (a) Soft-k L1 norm metric (Section \ref{subsec:setup}) as the number of trajectories $J$ increases. (b) Runtime comparison averaged 9 runs (c) Soft-k L1 metric for WHIRL with increasing arms.}
    \label{fig:baselines}
\end{figure*}

\paragraph{Computation Cost and Backpropagation}\label{sec:computation-cost}
The computation time for the soft top k Whittle policy is $O(N)$ \cite{xie2020differentiable}, and there is a matrix inversion required when differentiating through $\frac{dW}{d\tilde{\mathbf{R}}}$ which has a runtime of $O(M^\omega)$ where $\omega$ is the matrix conversion constant. Therefore, the overall computation of $N$ arms and $M$ states for a gradient-based update is $O(N M^{\omega})$ per gradient step. In contrast, existing IRL algorithms would construct a large joint-state MDP's with computation cost $O(M^N)$. 
\section{Experiments}
We first  show  WHIRL is more accurate and scalable than existing IRL baselines \footnote{WHIRL converges at 30 epochs. We tune the parameter $\epsilon$ where  $\epsilon=0$ is a regular top k and $\epsilon=1$ is a completely random policy. We saw best results with $\epsilon=0.01$, $\alpha = 0.01$ for our Adam optimizer, and discount factor $\gamma=0.99$}. We use a synthetic dataset for this as it allows direct comparison with IRL baselines that take expert trajectories as inputs, and to compare how close WHIRL can get to the original true policy, which we only have access to in synthetic settings. Next we show how WHIRL achieves desired public health outcomes in the real world MCH setting. 

\subsection{Synthetic Dataset: Setup and results}
\label{subsec:setup}
\paragraph{Setup} For the synthetic experiments, we generate $\trajectories^\text{expert}$ from a soft-top-k Whittle policy  on randomly generated  $\mathbf{R}$ and  $\mathbf{P(s,a,s')}$ for all arms (instead of generating from Alg \ref{alg:population-edits}). Transition probabilities have an additional constraint that pulling the arm ($a=1$) is strictly better than not pulling the arm ($a=0$) to ensure benefit of pulling. While learning, we only have access to  $\mathbf{P(s,a,s')}$ and $\trajectories^\text{expert}$. We first consider a RMAB problem composed of $N = 2$ arms, $M=2$ states, budget $K=1$, and time horizon $H=3$. We use this small problem to compare WHIRL with baselines because the latter become computationally intractable for many arms. The results are averaged over 64 runs. For WHIRL only, we also consider RMAB settings composed of $N$ = 50, 100, and  200, $M=2$, $K=20$, and $T=10$. These results are averaged over 48 runs. We compare WHIRL with (1) Max Entropy IRL and (2) MAP-B-IRL.\footnote{Code is used from \cite{suresh2022marginal} and max-ent python library} To evaluate performance we compute the L1 norm between the learned soft-k and expert soft-k policy probabilities $\lVert \policy^{\text{soft}}_W(R^\text{expert}) -\policy^{\text{soft}}_W(R^\text{learner}) \rVert_1$ normalized by $N * J * H$. This metric allows us to test whether the policy from the learned rewards mimics the decision-making used to create $\trajectories^\text{expert}$. Such a metric is infeasible with the real dataset, as the true rewards are not available. 

\paragraph{Experimental Results}\label{sec:results}
Figure~\ref{fig:baselines}a shows that WHIRL performs significantly better than the baselines, maximum entropy and MAP-BIRL, as we increase trajectories. All three algorithms have a slight decrease in performance after 5 trajectories, but that effect goes away as we increase the horizon size. We attribute this to an overfitting of reward learning with smaller data samples which improves with more data. Figure~\ref{fig:baselines}c shows that WHIRL can learn rewards leading to a near-optimal policy on larger problems of $N=50$, 100, and 200 with just 3 input trajectories. The $\policy^{\text{soft}}_W(R^\text{expert})$ will not be available for real-world data, so this experiment validates that WHIRL can learn a nearly optimal policy with very few training trajectories on large problem sizes. Finally, Figure \ref{fig:baselines}b compares the computation cost per gradient step of WHIRL versus the two baselines by changing $N$ in the $M=2$-state RMAB problem. Both baselines grow exponentially w.r.t. N and become unusable past 6 arms. 

\subsection{Maternal Health dataset: Setup and results} \label{applied}
We also analyze our algorithm on an MCH dataset provided by Armman from a deployment in 2021 in which each beneficiary was modeled as an RMAB arm with a naive \{0,1\} reward. This data records beneficiary listening and beneficiaries selected for human calls in the form $(s,a)$ for 10 weeks of the program. $s=1$ if the beneficiary listened to over 30 seconds of a call, and $a=1$ if the beneficiary was intervened on in that timestep. \footnote{The dataset originally contains 7668 beneficiaries but Armman does not intervene on mothers who are already high listeners, so only 2127 mothers are eligible for calls.} The dataset contains 2127 Armman beneficiaries from the city of Mumbai. They are fully anonymized. We also have access to 13 features about each beneficiary including demographic features like income, education, and language, and also health information like gestational age, previous number of children/miscarriages. Sensitive features such as
caste and religion are not collected. Armman sends out 71 human calls per timestep for the 2127 beneficiaries, and we estimate beneficiary transition data by aggregating each beneficiary's listening patterns over the 10 timesteps and clustering them \cite{mate2022field}. Lastly, beneficiaries must consent to be in the program. 
\subsubsection{Setup} WHIRL is part of an overall system (Figure \ref{fig:system}) which works as follows: \textbf{Step 1:} Given execution of the current RMAB policy over the last H timesteps (H=10 weeks in our experiments), the health expert (in this case, a service call manager) sees aggregated statistics of which categories of beneficiaries obtained service calls (e.g. Figure \ref{fig:before_after_actions_risk}a). For this step, we show a diverse set of aggregate statistics including actions based on demographic information (i.e. income, education, language, phone ownership) but also dynamic information like what states people are at when they are called. The service call manager gives aggregate level feedback based on these statistics. \textbf{Step 2:} The feedback is used to create $\trajectories^\text{expert}$ from which WHIRL generates new rewards. \textbf{Step 3:} Using a new policy obtained from WHIRL rewards, the system conducts a what-if analysis by synthetically simulating future trajectories for the original vs WHIRL rewards, and compares them to show the difference in service call allocations and impact on listening to AVMs by using WHIRL (Figure \ref{fig:risk_absolute_change}). \textbf{Step 4:} If the service call manager approves, Armman should use the new WHIRL rewards  to plan actions for the same beneficiaries for timesteps H+1 onwards. If not, the process repeats from step 2 with new feedback from the service call manager. We will now analyze 2 examples in depth, and also provide a systematic evaluation of WHIRL (Table \ref{tab:systematic}). 

\subsubsection{Risk-Based Rewards} \label{sec:risk_rewards} To characterize beneficiaries, Armman uses risk score $= \boldsymbol 1_{\text{education level} < \text{threshold}} + \boldsymbol 1_{\text{income} < \text{threshold}} + \boldsymbol 1_{\text{phone ownership} = \text{False}}$, where thresholds are set by Armman. Armman has previously found that these scores highly correlate with eventual health outcomes. We notice that a large percentage of currently called beneficiaries come from risk score 0 and 1 (Figure \ref{fig:before_after_actions_risk}a), and yet those beneficiaries are not ideal targets for interventions because they already have access to many other resources, so the lift from Armman's intervention is very low. 
We apply WHIRL to this use case, using the aggregate feedback: \textit{Move interventions from risk groups {0,1} to {2,3}}. The actions from the original $\trajectory$ and $\trajectories^\text{expert}$ are shown in Figure \ref{fig:before_after_actions_risk}. Most actions are moved from low \{0,1\} to high \{2,3\} risk beneficiaries\footnote{Not all actions are always moved because there is built in entropy in Algorithm \ref{alg:population-edits}}. We then simulate the RMAB policy for the following 10 timesteps (averaged over 60 runs) for WHIRL rewards versus the \{0,1\}. WHIRL rewards significantly shift interventions toward high-risk beneficiaries (Figure \ref{fig:risk_absolute_change}b). We also see a bump in the number of AVM calls heard (``listenership'') for risk scores 2 and 3 (Figure \ref{fig:risk_absolute_change}b). The listenership and actions for groups 1 and 2 also decrease, which may be an acceptable trade-off since Risk 2 and 3 score mothers have more to gain from the calls according to Armman.

\begin{figure}[htbp]
    \centering
    \begin{minipage}{0.48\textwidth}
        \centering
        \includegraphics[width=\linewidth]{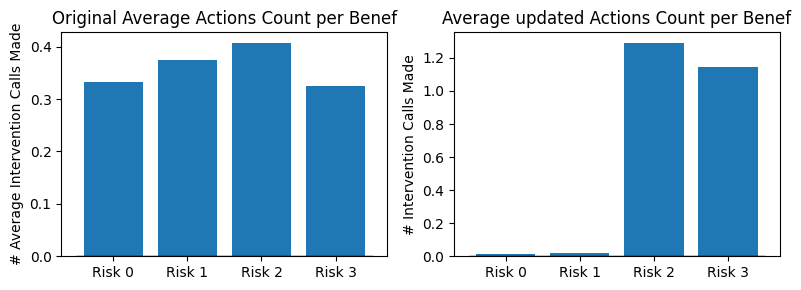}
        \caption{The average number of actions taken on beneficiaries in each risk group before and after the aggregate feedback (Alg \ref{alg:population-edits}).} 
        \label{fig:before_after_actions_risk}
    \end{minipage}
    \hfill
    \begin{minipage}{0.48\textwidth}
        \centering
        \includegraphics[width=\linewidth]{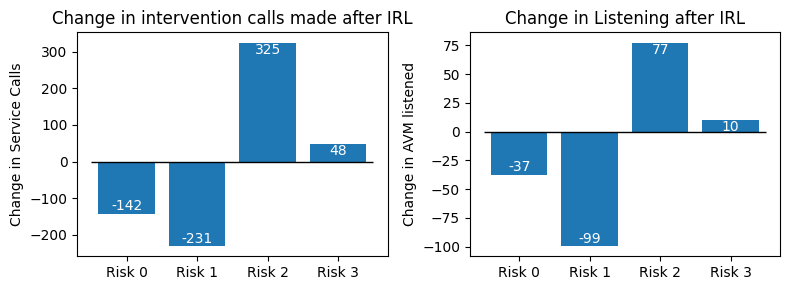}
        \caption{What-if analysis: Difference in (a) interventions (b) AVM calls listened to across risk between original and WHIRL rewards.} 
        \label{fig:risk_absolute_change}
    \end{minipage}
\end{figure}
A natural question at this point is to understand whether we could have hand crafted rewards to achieve the same effect. We try an intuitive assignment where we give all low-risk beneficiaries a reward of [0.3,0.3] for their engaging and disengaging state. This should make them unattractive to pull since they receive the same reward from both states. We keep high-risk beneficiaries at a reward of [0,1] where 0 is for the disengaged state and 1 is for the engaged state. This will make them attractive to keep engaged. After simulating hand crafted and WHIRL over 60 runs, results in Figure \ref{fig:action_distribution}a show that the hand crafted rewards distribute actions extremely unfairly where over 60\% (388 out of 597) of the high-risk beneficiaries have a 0\% probability of ever being called. WHIRL on the other hand calls beneficiaries much more evenly, only leaving 7 beneficiaries with a 0\% probability (Figure \ref{fig:action_distribution}b). This directly addresses our original goal of wanting evenly redistribute interventions to those who need it most. If we only relied on hand-crafted rewards we would still end up with unfair policies that try to maximize listening without taking into account how interventions are allocated across the group.
\begin{figure}
    \centering
    \includegraphics[width=.9\linewidth]{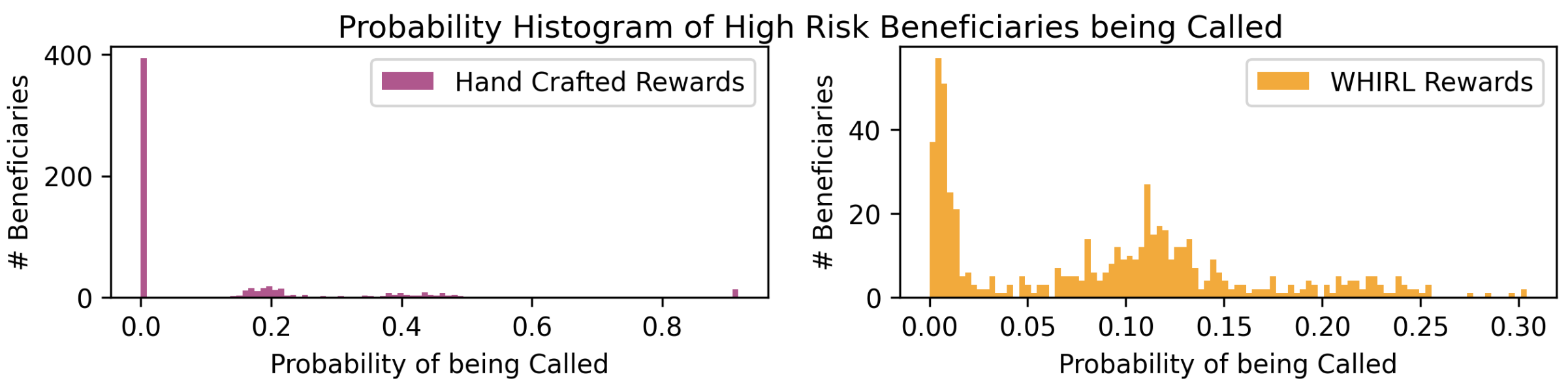}
    \caption{Histogram of high risk beneficiaries and their probabilities of intervened on.} 
    \label{fig:action_distribution}
\end{figure}
\subsubsection{Learning Rewards in Complex Settings}
We consider a more challenging behavior in which beneficiaries listen to at least 1 AVM in the preceding $T$ weeks. There is repetition of content in the AVM \cite{podcast}, and it is desirable to maximize the amount of \emph{distinct} content heard by beneficiaries rather than total calls.The state space increases from $2^1$ to $2^T$. 
It is more difficult to define reward functions as the state space increases, but it could make sense to call beneficiaries that have not heard any of the recent messages. Consider $T=3$, a setting with 8 states: \{000,001,010,011,100,101,110,111\}. The third digit is the most recent week's listening state. State 000 is the worst state. State 100 is problematic because a beneficiary may transition to 000 where they have not listened to any of the last 3 AVMs (state 010 is slightly better, but still close to reaching the 000 state). To prevent such situations, the service call manager uses the command: \textit{Move all service calls from states \{111,101,110,011,001\} to \{000,100,010}\}. 

After using WHIRL, actions change corresponding to the desired behavior (Figure \ref{fig:8_state}a), but we also see some interesting side effects. The number of interventions in state 110 went up even though that was not specified in the command. This is likely because the next passive state for 110 is 100 and WHIRL seems to have figured out that is an undesirable state and should be avoided. In terms of listening (Figure~\ref{fig:8_state}b), visitation in states 000 and 111 decreases. As intervention actions in 000 increased, the number of beneficiaries in 000 decreased. Perfect listening is not required, so service calls in 111 went down and hence the number of beneficiaries in the state decreased. Second, as desired, visitation for all of the middle states (i.e., states where 1 or 2 of the last 3 AVMs were listened to) increased. 
Visits to 001 increased because of increased interventions in its neighboring state 000, causing transition to 001. 
We lastly stress test our algorithm with diverse commands that capture more of Armman's goals (Table \ref{tab:systematic}). Even when we make the commands more complex with multiple conditions, we are able to make corresponding improvements in listening behavior. These cases are impossible for a human to manually design rewards for, and yet WHIRL is able to align its rewards with expert preferences. 

\begin{figure}
    \centering
    \includegraphics[width=.7\linewidth]{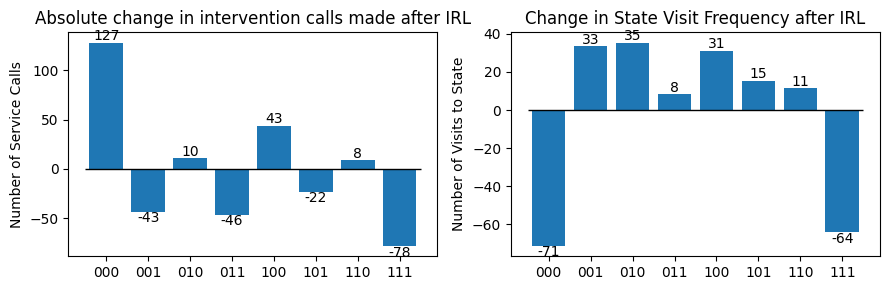}
    \caption{Difference in (a) actions, (b) state visits for 8 state case for naive vs WHIRL rewards} 
    \label{fig:8_state}
\end{figure}

\begin{table*}
  \centering
  \caption{WHIRL on a diverse set of commands. All results are averaged over 60 runs.}
  \label{tab:systematic}
  \begin{tabular}{@{}p{0.4\linewidth}@{\hspace{.3cm}}p{0.6\linewidth}@{}}
    \toprule
    \textbf{Command} & \textbf{Result} \\
    \midrule
    1. Move no more than 30 interventions/ week from low risk to high risk & 1. \textbf{18\%} increase in calls, \textbf{0.3\%} increase in listening to high risk groups, but smaller change than Fig \ref{fig:risk_absolute_change}. \\
    2. Move interventions from state 01,11 to state 10,00 but only from Marathi  to non-Marathi speakers. & 2a. \textbf{9\%} increase in calls to 00, 10 state, \textbf{0.5\%} increase in listening in 01, 10 state. 2b. \textbf{16\%} increase in calls and \textbf{0.4\%} increase in listening for non-Marathi speakers. No decrease in calls to non-Marathi speakers in states 01,11. \\
    3. Move interventions from low to high values of P(0,1,1) - P(0,0,1) & 3. \textbf{260\%} increase in calls and \textbf{11\%} increase listening for beneficiaries with higher gaps in transition probability \\
    4. Move interventions from highly educated and women phone owners to low educated and family/husband owned phone owners & 4a. \textbf{25\%} increase in calls and \textbf{0.4\%} increase in listenership for lower education mothers. 4b. \textbf{54\%} increase in calls and \textbf{1\%} increase in listenership for mothers who don’t own their phones. \\
    \bottomrule
  \end{tabular}
\end{table*}

\section{Discussion and Deployment} 
Our experiments demonstrate a need for techniques such as WHIRL since desirable outcomes are not easy to define in the RMAB reward model and rewards can only be learned subject to resource constraints and uncertain beneficiary transitions. WHIRL enables the human expert to view the implications of new desired behavior before utilizing the new rewards. Additionally, WHIRL's ability to incorporate large-scale aggregate feedback enabled the creation of 2,127 expert trajectories, making it feasible to use IRL in public health settings in a way that was not previously possible. In India, outcomes for mobile health programs vary significantly based on state-level advocacy and cultural norms \cite{mohan_optimising_2022}, so we expect service call managers in different states to have different aggregate feedback. As a result, the system will require testing with more service call managers across different states and better understand how our tool performs for differing goals. We publicly release our code at \url{https://github.com/Gjain234/WHIRL} to facilitate future work. While there is still more to be done before deployment, our work paves the way for a new feedback-driven use of RMABs for large-scale resource allocation.

\subsection{Ethical Discussion} \label{ethical_statement}
Resource allocation implicitly brings ethical challenges with it. We are learning new rewards in a resource constrained environment, so people that initially were assigned to receive interventions perhaps will not receive them anymore. However, the motivation of our project in the first place was to better align allocation of resources to expert preferences, and we show that with WHIRL we are able to send more calls to high risk beneficiaries (Figure \ref{fig:risk_absolute_change}). Another concern may be that the expert gives an incorrect directive; to safeguard against that, we rely on the what-if analysis, where we show the expert summary statistics about how the new rewards will affect who is called, and the expert can decide to go forward with those rewards (Section \ref{applied}). 

\section{Acknowledgements}
Gauri Jain is supported by the National Science Foundation under Grant No. IIS-2229881. Pradeep Varakantham is supported by the Lee Kuan Yew Fellowship fund awarded by Singapore Management University. We thank Prasanth Suresh for his support in working with MAP-BIRL. We lastly thank Neha Madhiwalla and the entire Armman team for their advice and feedback throughout the project. 

\bibliographystyle{splncs04}
\bibliography{references}
\end{document}